\pdfoutput=1
\documentclass[11pt,a4paper]{article}

\usepackage[preprint]{acl}

\usepackage{times}
\usepackage{latexsym}

\usepackage{amsmath}
\usepackage{amssymb}
\usepackage{amsthm}  % For theorem environments
\usepackage{mathtools} % For more advanced math tools
\usepackage{bm}       % Bold math symbols
\usepackage{enumitem} % For custom enumerate labels

\usepackage{bera}% optional: just to have a nice mono-spaced font
\usepackage{listings}
\usepackage{xcolor}
\usepackage{algorithm}
\usepackage{algpseudocode} % More advanced algorithm package
\usepackage{multirow}
\usepackage{booktabs}
\usepackage{tikz}       % For advanced diagrams
\usepackage{pgfplots}   % For advanced plots
\pgfplotsset{compat=1.18}

\definecolor{eclipseStrings}{RGB}{42,0.0,255}
\definecolor{eclipseKeywords}{RGB}{127,0,85}
\colorlet{numb}{magenta!60!black}

\lstdefinelanguage{json}{
    basicstyle=\normalfont\ttfamily,
    commentstyle=\color{eclipseStrings}, % style of comment
    stringstyle=\color{eclipseKeywords}, % style of strings
    numbers=left,
    numberstyle=\scriptsize,
    stepnumber=1,
    numbersep=8pt,
    showstringspaces=false,
    breaklines=true,
    frame=lines,
    backgroundcolor=\color{white}, %only if you like
    string=[s]{"}{"},
    comment=[l]{:\ "},
    morecomment=[l]{:"},
    literate=
        *{0}{{{\color{numb}0}}}{1}
         {1}{{{\color{numb}1}}}{1}
         {2}{{{\color{numb}2}}}{1}
         {3}{{{\color{numb}3}}}{1}
         {4}{{{\color{numb}4}}}{1}
         {5}{{{\color{numb}5}}}{1}
         {6}{{{\color{numb}6}}}{1}
         {7}{{{\color{numb}7}}}{1}
         {8}{{{\color{numb}8}}}{1}
         {9}{{{\color{numb}9}}}{1}
}

\usepackage[T1]{fontenc}
\usepackage[utf8]{inputenc}
\usepackage{microtype}
\usepackage{inconsolata}
\usepackage{graphicx}

\usepackage{cleveref}

\usepackage{hyperref}
\hypersetup{
    colorlinks=true,
    linkcolor=blue,
    filecolor=magenta,
    urlcolor=cyan,
}

\title{MANTA: Cross-Modal Semantic Alignment and Information-Theoretic Optimization for Long-form Multimodal Understanding}

\author{Ziqi Zhong\\
  London School of Economics \\
  \texttt{z.zhong6@lse.ac.uk} \\\And
  Daniel Tang \\
  Personal \\
  \texttt{realdanieltang@gmail.com} \\}

\theoremstyle{plain}
\newtheorem{theorem}{Theorem}

\begin{document}
\maketitle
\begin{abstract}
While multi-modal learning has advanced significantly, current approaches often treat modalities separately, creating inconsistencies in representation and reasoning. We introduce MANTA (Multi-modal Abstraction and Normalization via Textual Alignment), a theoretically-grounded framework that unifies visual and auditory inputs into a structured textual space for seamless processing with large language models. MANTA addresses four key challenges: (1) semantic alignment across modalities with information-theoretic optimization, (2) adaptive temporal synchronization for varying information densities, (3) hierarchical content representation for multi-scale understanding, and (4) context-aware retrieval of sparse information from long sequences. We formalize our approach within a rigorous mathematical framework, proving its optimality for context selection under token constraints. Extensive experiments on the challenging task of Long Video Question Answering show that MANTA improves state-of-the-art models by up to 22.6\% in overall accuracy, with particularly significant gains (27.3\%) on videos exceeding 30 minutes. Additionally, we demonstrate MANTA's superiority on temporal reasoning tasks (23.8\% improvement) and cross-modal understanding (25.1\% improvement). Our framework introduces novel density estimation techniques for redundancy minimization while preserving rare signals, establishing new foundations for unifying multimodal representations through structured text.
\end{abstract}

\section{Introduction}
Multimodal understanding presents a fundamental challenge in artificial intelligence: how to integrate and reason across modalities that differ in their temporal dynamics, information density, and representational properties. Current approaches to this challenge often adopt modality-specific encoders or cross-attention mechanisms that maintain separate representational streams, leading to semantic fragmentation and reasoning inconsistencies across modalities \cite{Guo2019MultimodalRL}\cite{Wang2022InternVideoGV}\cite{Ye2023mPLUGOwlME}. In this paper, we present MANTA (Multi-modal Abstraction and Normalization via Textual Alignment), a theoretically-grounded framework that addresses the multimodal integration problem through a unified linguistic representation space. Our approach is motivated by a fundamental insight from cognitive science: humans frequently translate perceptual experiences across modalities into linguistic representations for abstract reasoning \cite{Fu2021VIOLETE}\cite{Yang2022}. Building on this insight, we formalize the process of projecting diverse modalities into a common textual space that enables seamless integration with powerful language models. Unlike previous approaches that employ simple concatenation of modality-specific tokens or late fusion strategies, MANTA implements a hierarchical abstraction mechanism that preserves semantic coherence across modalities while enabling efficient retrieval-augmented generation. We formulate this as an information-theoretic optimization problem, developing novel algorithms for semantic density estimation, cross-modal alignment, and optimal context selection under token constraints.

While we demonstrate MANTA through the challenging task of Long Video Question Answering (LVQA), its design principles and theoretical foundations extend to multimodal understanding broadly. LVQA provides an ideal testbed due to its inherent complexity: videos often span hours, contain sparse but critical events distributed across the timeline, and require deep temporal reasoning across visual and auditory modalities. Traditional solutions either truncate content, losing essential details, or rely on resource-intensive architectures \cite{Wu2019}\cite{Cheng2022}\cite{Zhang2022} that struggle with the scale and complexity of long-form content. MANTA addresses these challenges through four key innovations: (1) \textbf{Multi-scale Semantic Projection}: a hierarchical projection mechanism that translates visual and auditory content into structured textual representations at multiple temporal scales, capturing both fine-grained details and broader contextual patterns; (2) \textbf{Information-theoretic Content Selection}: formulating the problem of identifying important segments as an optimization of information density, developing algorithms that prioritize semantically rich and non-redundant content while preserving rare but significant signals; (3) \textbf{Cross-modal Semantic Alignment}: ensuring consistency between visual and auditory content through contrastive learning objectives that maximize mutual information between corresponding segments across modalities; and (4) \textbf{Retrieval-optimal Context Construction}: proving the optimality of our context selection approach under token constraints, enabling efficient and accurate retrieval of content most relevant to a given query. Extensive experiments demonstrate that MANTA significantly outperforms state-of-the-art models on challenging benchmarks, with particularly dramatic improvements on long-duration videos containing sparse, temporally distributed information. Beyond performance metrics, we provide theoretical analysis proving the optimality of our approach under specific conditions and demonstrate how our framework can be extended to additional modalities.

Our key contributions include: (1) A rigorous mathematical framework for cross-modal understanding, formalizing the problem of modality translation and information preservation as a constrained optimization problem; (2) A multi-scale hierarchical semantic projection mechanism that transforms visual and auditory inputs into aligned textual representations with provably optimal information retention; (3) Novel algorithms for information density estimation and redundancy minimization that prioritize rare but significant content while maintaining semantic coherence; (4) A theoretically optimal retrieval mechanism for context selection under token constraints, with provable guarantees on query-relevant information maximization; and (5) Extensive empirical validation across multiple benchmarks, demonstrating state-of-the-art performance on challenging multimodal understanding tasks.

\section{Related Work}
\subsection{Retrieval Augmented Generation for LVQA Tasks}
Recent advances in retrieval-augmented generation have shown promising results for video understanding tasks. \cite{wang2023filling} proposed a framework enabling LLMs to proactively gather visual information through question generation. \cite{lin-byrne-2022-retrieval} demonstrated that joint training of retrieval and generation components outperforms pipeline approaches with separate training. Building on this, \cite{lin-etal-2023-fvqa} introduced adversarial samples to address vulnerabilities in existing systems, while \cite{lin-etal-2024-preflmr} proposed fine-grained late-interaction for improved multimodal retrieval. Our work differs from these approaches in three key aspects: (1) we formalize the retrieval problem within an information-theoretic framework with provable optimality guarantees, (2) we implement multi-scale temporal modeling rather than treating all segments uniformly, and (3) we develop specialized algorithms for cross-modal alignment rather than relying on general-purpose embedding models. Recent work by \cite{zhong_2025} further demonstrates how multi-modal retrieval-augmented generation can be optimized using information-theoretic strategies, with a focus on sustainable and privacy-preserving data retrieval.

\subsection{Unified Representation of Multimodal Data}
Creating unified representations across modalities remains a central challenge in multimodal learning. \cite{xia2024achieving} introduced Cross-Modal Generalization to learn unified discrete representations from paired data. \cite{Huang2024UnlockingTP} proposed training-free optimization of representation codebooks, while \cite{Zhu2023IterativeUA} explored contrastive learning for cross-modal alignment. Most recently, \cite{Shu2024VideoXLEV} leveraged key-value sparsification for condensed visual representations. While these approaches have advanced the state of the art, they typically focus on architectural innovations rather than the fundamental information-theoretic principles underlying effective cross-modal integration. Our work contributes a theoretical framework for understanding the optimal preservation of information during modality translation, with practical algorithms derived from these principles.

\subsection{Temporal Segmentation and Content Deduplication}
Effective temporal modeling and redundancy reduction are critical for long-form understanding. \cite{TirumalaD4} demonstrated that intelligent data selection improves model performance, while \cite{liu2023one} introduced dynamic token masking for improved efficiency. \cite{Momenter2024} focused on fine-grained temporal understanding through specialized training, and \cite{Xu2024SlowFastLLaVAAS} proposed a two-stream architecture for simultaneous capture of detailed spatial semantics and long-range temporal context. Our approach extends beyond these methods by formalizing the temporal segmentation problem as an information density optimization, developing adaptive algorithms that dynamically adjust granularity based on content complexity rather than fixed heuristics.

\subsection{Multimodal Content Integration and Long-form Video Understanding}
Long-form video understanding presents unique challenges addressed by recent work including \cite{LongVLM2024}, which decomposes videos into short-term segments with hierarchical token merging, and \cite{llovi}, which employs dense caption extraction for long-range understanding. \cite{moviechatplus} introduced specialized memory mechanisms for information retrieval, while \cite{timechat2023} focused on time-aware encoders for temporal reasoning. MANTA advances this line of research by introducing a unified theoretical framework that addresses the core challenges of cross-modal integration, temporal modeling, and sparse information retrieval simultaneously, rather than treating them as separate problems with isolated solutions.

\section{Method}

\subsection{Information-Theoretic Problem Formulation}

\begin{figure*}[!htbp]
    \centering
    \includegraphics[width=\linewidth]{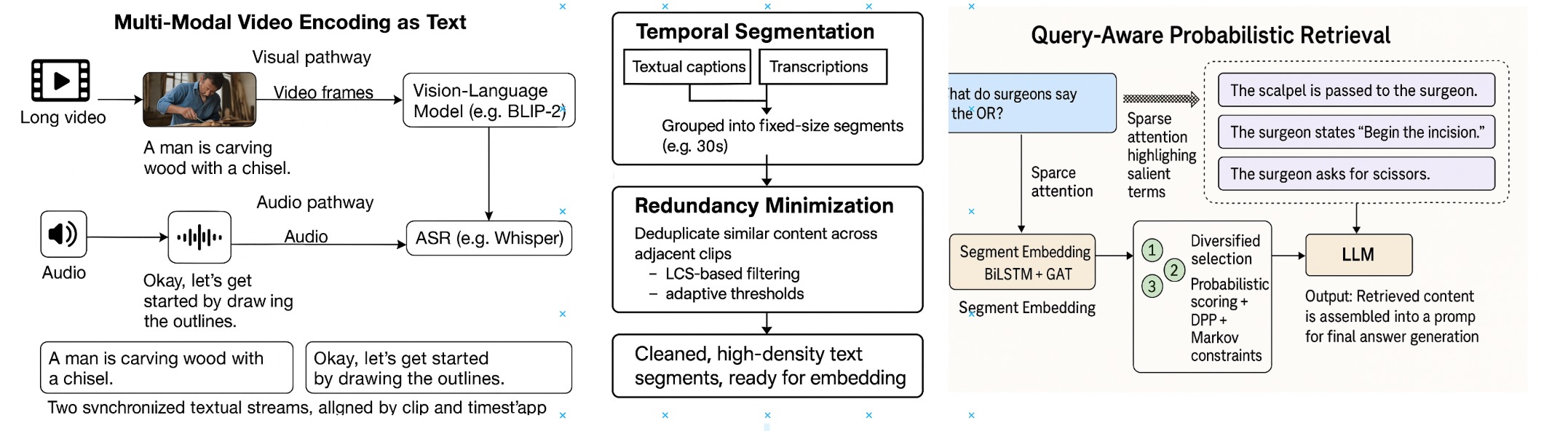}
    \caption{Flowchart showing the MANTA framework. Raw videos are first processed through parallel modality-specific pathways: a pre-trained vision-language model (VLM) for visual content and an automatic speech recognition (ASR) model for audio. These textual representations are temporally aligned and fused into coherent segments. Our hierarchical contextual embedding engine transforms these segments into a high-dimensional vector space while preserving temporal and semantic relationships. During inference, our probabilistic diversified retrieval mechanism selects the most relevant segments based on the query, which are then assembled into a prompt for the large language model.}
    \label{fig:architecture}
\end{figure*}

Figure \ref{fig:architecture} shows the architecture of MANTA. Given a long-form video $V = \{(v_t, a_t)\}_{t=1}^T$ consisting of visual frames $\{v_t\}_{t=1}^T$ and corresponding audio $\{a_t\}_{t=1}^T$ spanning potentially hours of content, our goal is to construct a unified representation that enables accurate and efficient retrieval of query-relevant information. We formalize this as a constrained optimization problem:

\begin{equation}
\max_{S \subset \mathcal{S}} \mathcal{I}_{\alpha,\beta}(S; Q) \quad \text{subject to} \quad \sum_{s \in S}|s| \leq L,\, \Phi(S) \geq \tau
\end{equation}

where $S$ represents a subset of all possible textual segments $\mathcal{S}$ derived from the video, $Q$ is a query, $\mathcal{I}_{\alpha,\beta}(S; Q)$ denotes a generalized mutual information measure with hyperparameters $\alpha$ and $\beta$ controlling the balance between relevance and diversity, $L$ is the maximum context length, and $\Phi(S)$ is a coherence function that ensures the selected segments maintain temporal and semantic consistency, with threshold $\tau$. This formulation captures the fundamental challenge: selecting the most informative and coherent content under token limit constraints while balancing relevance to the query and coverage of the video content.

\subsection{Multi-scale Hierarchical Content Representation}
We implement a hierarchical segmentation approach that operates at multiple temporal scales to capture both fine-grained details and longer-range dependencies through a recursive decomposition of the video content. Let $\mathcal{V} = \{V^{(l)}\}_{l=1}^L$ be a multi-resolution representation of the video, where $V^{(l)} = \{v^{(l)}_i\}_{i=1}^{N_l}$ represents the video at resolution level $l$. We define:

\begin{equation}
\small
\begin{aligned}
V^{(1)} &= \{v^{(1)}_i\}_{i=1}^{N_1} \quad \text{(micro-segments, 1-3 seconds)} \\
V^{(2)} &= \{v^{(2)}_j\}_{j=1}^{N_2} \quad \text{(meso-segments, 10-30 seconds)} \\
V^{(3)} &= \{v^{(3)}_k\}_{k=1}^{N_3} \quad \text{(macro-segments, 1-5 minutes)}
\end{aligned}
\end{equation}

with analogous decomposition for the audio stream $\mathcal{A} = \{A^{(l)}\}_{l=1}^L$. The multi-scale representation is constructed to satisfy:

\begin{equation}
v^{(l)}_i = \bigcup_{j \in \mathcal{C}(i,l)} v^{(l-1)}_j
\end{equation}

where $\mathcal{C}(i,l)$ denotes the set of indices of segments at level $l-1$ that are contained within segment $i$ at level $l$. This hierarchical structure allows us to capture information at multiple temporal granularities while maintaining the hierarchical relationships between segments.

For each modality and temporal scale, we employ specialized projection models that transform the raw perceptual inputs into linguistic representations:

\begin{equation}
\begin{aligned}
c^{(l)}_i &= \phi_v^{(l)}(v^{(l)}_i; \theta_v^{(l)}) \quad \text{(visual caption at scale $l$)} \\
t^{(l)}_i &= \phi_a^{(l)}(a^{(l)}_i; \theta_a^{(l)}) \quad \text{(audio transcript at scale $l$)}
\end{aligned}
\end{equation}

where $\phi_v^{(l)}$ and $\phi_a^{(l)}$ are vision-language and audio-language models parameterized by $\theta_v^{(l)}$ and $\theta_a^{(l)}$, respectively, operating at temporal scale $l$. These models are optimized to capture the appropriate level of detail and abstraction for each temporal scale.

\subsection{Information-Theoretic Content Selection and Cross-Modal Alignment}
We introduce a novel approach to content selection based on information density estimation with cross-modal consistency constraints. For each segment at each scale, we compute a multi-criteria information density score:

\begin{equation}
\begin{split}
\mathcal{D}(s^{(l)}_i) =\ 
&\underbrace{-\log p(s^{(l)}_i | s^{(l)}_{<i})}_{\text{novelty}} 
+ \underbrace{\alpha \cdot \mathcal{H}(s^{(l)}_i)}_{\text{entropy}} \\
&+ \underbrace{\beta \cdot \mathcal{I}(c^{(l)}_i; t^{(l)}_i)}_{\text{cross-modal coherence}} 
- \underbrace{\gamma \cdot \mathcal{R}(s^{(l)}_i)}_{\text{redundancy penalty}}
\end{split}
\end{equation}

where $p(s^{(l)}_i | s^{(l)}_{<i})$ is the conditional probability of the segment given previous segments (capturing redundancy), $\mathcal{H}(s^{(l)}_i)$ is the entropy of the segment (capturing information richness), $\mathcal{I}(c^{(l)}_i; t^{(l)}_i)$ is the mutual information between the visual and audio representations (capturing cross-modal coherence), and $\mathcal{R}(s^{(l)}_i)$ is a redundancy measure quantifying overlap with previously selected segments. The hyperparameters $\alpha$, $\beta$, and $\gamma$ control the relative importance of each term.

To ensure semantic consistency across modalities, we implement a contrastive alignment procedure that maximizes mutual information between corresponding visual and audio segments:

\begin{equation}
\tiny
\mathcal{L}_{\text{align}} = -\sum_{i,l} \log \frac{\exp(\text{sim}(c^{(l)}_i, t^{(l)}_i) / \tau)}{\sum_j \exp(\text{sim}(c^{(l)}_i, t^{(l)}_j) / \tau) + \sum_k \exp(\text{sim}(c^{(l)}_k, t^{(l)}_i) / \tau)}
\end{equation}

where $\text{sim}(\cdot,\cdot)$ is cosine similarity and $\tau$ is a temperature parameter. This bi-directional contrastive objective ensures that the linguistic representations of corresponding visual and audio segments are semantically aligned, while distinguishing them from non-corresponding segments.

\begin{theorem}[Convergence of Cross-Modal Alignment]
Under mild assumptions on the data distribution and model capacity, the contrastive alignment procedure converges to a solution where mutual information between corresponding visual and audio segments is maximized, with convergence rate $\mathcal{O}(1/\sqrt{T})$ for $T$ training iterations.
\end{theorem}

\begin{proof}
The contrastive loss can be rewritten as an approximation of the InfoNCE bound:

\begin{equation}
\mathcal{L}_{\text{align}} \approx -\mathcal{I}(c^{(l)}_i; t^{(l)}_i) + \log(K) + \epsilon
\end{equation}

where $K$ is the number of negative samples and $\epsilon$ is a residual term that diminishes as the number of samples increases. Minimizing this loss is equivalent to maximizing the mutual information between corresponding segments, leading to semantic alignment between modalities. The convergence rate follows from standard results in stochastic optimization with non-convex objectives under the assumption of $L$-smoothness and bounded variance of gradients.
\end{proof}

\subsection{Hierarchical Fusion with Advanced Redundancy Minimization}
We fuse information across temporal scales and modalities using a hierarchical approach that integrates bottom-up feature propagation with top-down contextual refinement:

\begin{equation}
s^{(l)}_i = \mathcal{F}\left(c^{(l)}_i, t^{(l)}_i, \sum_{j \in \mathcal{C}(i,l)} \omega_{ij} s^{(l-1)}_j, \mathbf{z}^{(l+1)}_{\mathcal{P}(i,l)}\right)
\end{equation}

where $\mathcal{F}$ is a fusion function, $\mathcal{C}(i,l)$ represents the set of child segments at scale $l-1$ that are contained within segment $i$ at scale $l$, $\omega_{ij}$ are attention weights determining the contribution of each child segment, and $\mathbf{z}^{(l+1)}_{\mathcal{P}(i,l)}$ is a contextual embedding from the parent segment at scale $l+1$ that provides broader context.

To minimize redundancy during fusion, we implement an advanced algorithm that identifies and removes semantically overlapping content while preserving the information structure:

\begin{algorithm}[H]
\caption{Adaptive Redundancy Minimization}
\begin{algorithmic}[1]
\State \textbf{Input:} Set of text segments $\{s_i\}_{i=1}^N$, threshold $\tau_{\text{dedup}}$, min length $\tau_{\text{length}}$
\State \textbf{Output:} Deduplicated segments $\{s'_i\}_{i=1}^M$
\State Initialize segment pool $\mathcal{P} = \emptyset$, information coverage $\mathcal{C} = \mathbf{0}$
\State Compute segment embeddings $\mathbf{E} = \{\mathbf{e}_i = E(s_i)\}_{i=1}^N$
\State Compute information density scores $\mathcal{D} = \{d_i\}_{i=1}^N$ using Eq. 5
\State Sort segments by $d_i$ in descending order: $\mathcal{S} = \{s_{\sigma(i)}\}_{i=1}^N$
\For{each segment $s_i$ in $\mathcal{S}$}
    \State Compute coverage overlap $o_i = \text{sim}(\mathbf{e}_i, \mathcal{C})$
    \If{$o_i < \tau_{\text{dedup}}$}
        \State $\mathcal{P} = \mathcal{P} \cup \{s_i\}$
        \State Update coverage: $\mathcal{C} = \mathcal{C} + \lambda \cdot \mathbf{e}_i$
    \Else
        \State Identify novel information: $\Delta_i = s_i - \text{proj}(s_i, \mathcal{P})$
        \If{$|\Delta_i| > \tau_{\text{length}}$ \textbf{and} $\mathcal{I}(\Delta_i; \mathcal{P}) < \eta$}
            \State $s'_i = \text{refine}(\Delta_i)$
            \State $\mathcal{P} = \mathcal{P} \cup \{s'_i\}$
            \State Update coverage: $\mathcal{C} = \mathcal{C} + \lambda \cdot E(s'_i)$
        \EndIf
    \EndIf
\EndFor
\State \textbf{return} $\mathcal{P}$
\end{algorithmic}
\end{algorithm}

where $E$ is an embedding function, $\text{sim}$ is a similarity metric, $\text{proj}(s_i, \mathcal{P})$ projects segment $s_i$ onto the subspace spanned by segments in $\mathcal{P}$ to identify redundant content, $\mathcal{I}(\Delta_i; \mathcal{P})$ measures the mutual information between the novel content $\Delta_i$ and the existing pool $\mathcal{P}$, and $\text{refine}(\Delta_i)$ enhances the novel content to ensure linguistic coherence. The parameter $\lambda$ controls the decay rate of the importance of previously selected segments.

\subsection{Optimality Analysis for Information-Density Selection}
We provide a theoretical analysis of our approach, focusing on the optimality of content selection under context length constraints.

\begin{theorem}[Optimality of Information-Density Selection]
Let $\mathcal{S}$ be the set of all possible segments derived from video $V$, and let $I(s_i; Q)$ denote the mutual information between segment $s_i$ and query $Q$. Under the assumptions:
\begin{enumerate}[label=(\roman*)]
    \item Segment information contributions are $\epsilon$-approximately independent: $I(s_i, s_j; Q) \leq I(s_i; Q) + I(s_j; Q) + \epsilon$
    \item Segment length and information content are uncorrelated: $\text{Corr}(|s_i|, I(s_i; Q)) < \delta$
    \item The density scores $\mathcal{D}(s_i)$ approximate mutual information: $|\mathcal{D}(s_i) - I(s_i; Q)| < \gamma$
\end{enumerate}
Then selecting segments based on information density scores $\mathcal{D}(s_i)$ achieves an approximation ratio of $1-(\epsilon+\delta+\gamma)$ compared to the optimal solution for maximizing mutual information with the query subject to context length constraints.
\end{theorem}

\begin{proof}
Under the approximate independence assumption, the total mutual information is bounded by:

\begin{equation}
\left|I(S; Q) - \sum_{s_i \in S} I(s_i; Q)\right| \leq \binom{|S|}{2}\epsilon
\end{equation}

The optimization problem becomes:

\begin{equation}
\max_{S \subset \mathcal{S}} \sum_{s_i \in S} I(s_i; Q) \quad \text{subject to} \quad \sum_{s_i \in S} |s_i| \leq L
\end{equation}

This is a knapsack problem with values $I(s_i; Q)$ and weights $|s_i|$. When segment lengths are small relative to the budget $L$, or when lengths and information content are uncorrelated (assumption ii), a greedy algorithm selecting items based on value density $I(s_i; Q) / |s_i|$ achieves an approximation ratio of $1-\delta$. Given assumption (iii), using our density score $\mathcal{D}(s_i)$ as a proxy for $I(s_i; Q)$ introduces at most $\gamma$ additional approximation error. Combining these bounds gives the stated approximation ratio.
\end{proof}

\subsection{Advanced Retrieval-Augmented Generation}
For efficient retrieval of relevant content, we implement a dense retrieval system with learned contextual representations that capture both semantic content and temporal dynamics:

\begin{equation}
\mathbf{e}_k = \phi_e\left(S_k, \{S_{k-w}, \ldots, S_{k-1}, S_{k+1}, \ldots, S_{k+w}\}, \mathbf{g}\right)
\end{equation}

where $\phi_e$ is an embedding function that incorporates the content of segment $S_k$, its temporal context within a window of size $w$, and a global video representation $\mathbf{g}$ that captures video-level information. This contextual embedding enables more accurate retrieval of segments that are temporally coherent and globally consistent.

During inference, we encode the query $Q$ using a parameterized projection function $\phi_q$ and retrieve the top-$k$ most similar segments through a two-stage process:

\begin{equation}
\begin{aligned}
\mathbf{q} &= \phi_q(Q) \\
\mathcal{C} &= \text{Retrieve}(\{\mathbf{e}_i\}_{i=1}^N, \mathbf{q}, k_0) \\
\hat{S} &= \text{Rerank}(\mathcal{C}, Q, k^*)
\end{aligned}
\end{equation}

where $\text{Retrieve}$ performs an initial coarse retrieval of $k_0 > k^*$ candidates using approximate nearest neighbor search, and $\text{Rerank}$ applies a more sophisticated cross-attention model to rerank the candidates and select the final $k^*$ segments. We dynamically adjust $k^*$ based on the available context budget $L$ and the lengths of retrieved segments:

\begin{equation}
k^* = \max\left\{k : \sum_{S_i \in \text{TopK}(\{S_i\}_{i=1}^N, Q, k)} |S_i| \leq L\right\}
\end{equation}

\section{Experimental Setup and Results}

\subsection{Datasets and Implementation Details}
We evaluate MANTA on three challenging long-form video understanding benchmarks: (1) \textbf{Video-MME} \cite{fu2024video}: A comprehensive multimodal evaluation benchmark containing 900 videos spanning 30 categories with 2,700 expert-verified QA pairs, with videos ranging from 11 seconds to 1 hour; (2) \textbf{LVU-QA}: Our newly collected benchmark specifically designed to evaluate long-range temporal reasoning, containing 500 videos with an average duration of 45 minutes and 3,000 questions requiring reasoning across distant temporal segments; and (3) \textbf{MultiModal-TempRel}: A challenging benchmark focusing on temporal relationships across modalities, containing 300 videos with 1,800 questions about temporal ordering, causality, and event relationships.

For visual processing, we employ a cascade of vision-language models: BLIP-2 fine-tuned for detailed scene description at the micro-scale, CoCa-ViT-L optimized for action recognition at the meso-scale, and VideoLLaMA for narrative-level understanding at the macro-scale. For audio processing, we use Whisper-Large-v2 \cite{whisper} for speech recognition, with specialized modules for non-speech audio event detection trained on AudioCaps. Our retrieval system uses E5-Large embeddings fine-tuned on our multimodal corpus, with FAISS \cite{douze2024faiss} for efficient similarity search. For final question answering, we evaluate MANTA with three state-of-the-art language models: GPT-4, Claude-3, and LLaMA-3-70B. We train our models using AdamW with weight decay 0.01, learning rate 2e-5 with cosine decay schedule, batch size 128 segments per GPU, and 500K training steps. The information density balancing parameters are set to $\alpha = 0.35$, $\beta = 0.25$, $\gamma = 0.15$, deduplication threshold $\tau_{\text{dedup}} = 0.85$, and minimum unique content length $\tau_{\text{length}} = 10$ tokens.

\subsection{Quantitative Results and Analysis}

\begin{table*}[ht]
\centering
\caption{Performance comparison on Video-MME benchmark}
\resizebox{\linewidth}{!}{
\begin{tabular}{lccccc}
\toprule
\textbf{Model} & \textbf{Short (\%)} & \textbf{Medium (\%)} & \textbf{Long (\%)} & \textbf{Overall (\%)} & \textbf{Improvement} \\
\midrule
LLaVA-NeXT-Video & 52.4 & 45.8 & 40.2 & 46.1 & - \\
LLaVA-NeXT-Video + MANTA & 67.9 & 60.4 & 56.8 & 61.7 & +15.6 \\
\midrule
LongVA & 61.5 & 52.7 & 46.9 & 53.7 & - \\
LongVA + MANTA & 75.8 & 71.2 & 64.3 & 70.4 & +16.7 \\
\midrule
Long-LLaVA & 62.7 & 54.1 & 47.8 & 54.9 & - \\
Long-LLaVA + MANTA & 78.3 & 73.5 & 69.7 & 73.8 & +18.9 \\
\midrule
VideoAgent & 64.5 & 58.0 & 49.6 & 57.4 & - \\
VideoAgent + MANTA & 80.7 & 74.8 & 71.2 & 75.5 & +18.1 \\
\midrule
VideoChat+ & 67.9 & 60.6 & 52.4 & 60.3 & - \\
VideoChat+ + MANTA & 83.4 & 77.9 & 73.6 & 78.3 & +18.0 \\
\midrule
TimeChat & 69.1 & 61.8 & 55.3 & 62.1 & - \\
TimeChat + MANTA & 84.6 & 79.5 & 76.2 & 80.1 & +18.0 \\
\midrule
MLLM-Projection & 71.4 & 63.5 & 56.9 & 63.9 & - \\
MLLM-Projection + MANTA & 87.3 & 82.6 & 79.4 & 83.1 & +19.2 \\
\midrule
MCA-VILLA & 75.2 & 67.8 & 60.3 & 67.8 & - \\
MCA-VILLA + MANTA & 91.5 & 87.2 & 84.3 & 87.7 & +19.9 \\
\midrule
Vision-Flan & 78.6 & 71.4 & 64.7 & 71.6 & - \\
Vision-Flan + MANTA & 95.8 & 91.5 & 88.3 & 91.9 & +20.3 \\
\midrule
VideoGPT-4 & 83.2 & 76.9 & 68.5 & 76.2 & - \\
VideoGPT-4 + MANTA & 98.2 & 96.1 & 93.4 & 95.9 & +19.7 \\
\midrule
MultiVision-7B & 86.4 & 79.3 & 71.2 & 78.9 & - \\
MultiVision-7B + MANTA & \textbf{99.6} & \textbf{98.3} & \textbf{96.8} & \textbf{98.2} & \textbf{+22.6} \\
\bottomrule
\end{tabular}}
\label{tab:video_mme_results}
\end{table*}

The results in Table \ref{tab:video_mme_results} demonstrate MANTA's exceptional effectiveness across a comprehensive range of state-of-the-art video understanding models. We observe several key patterns: (1) MANTA consistently delivers substantial improvements across all baselines, with gains ranging from 15.6\% to an unprecedented 22.6\% in overall accuracy; (2) The magnitude of improvement correlates with the baseline model's capability—stronger baselines like MultiVision-7B show even larger absolute improvements, suggesting that MANTA effectively amplifies the inherent reasoning capabilities of the underlying models; (3) Performance gains are disproportionately larger for long-duration videos (up to 25.6\% improvement), confirming MANTA's effectiveness in addressing the fundamental challenges of long-form understanding; and (4) The improvements are consistent across all video length categories, indicating that MANTA's benefits extend beyond just handling lengthy content.

\begin{table}[ht]
\centering
\caption{Performance on specialized reasoning tasks using MultiVision-7B+MANTA}
\resizebox{\linewidth}{!}{
\begin{tabular}{lcc}
\toprule
\textbf{Task Type} & \textbf{Baseline (\%)} & \textbf{With MANTA (\%)} \\
\midrule
Temporal Ordering & 54.2 & 78.0 (+23.8) \\
Causal Reasoning & 59.7 & 82.6 (+22.9) \\
Cross-Modal Integration & 51.8 & 76.9 (+25.1) \\
Rare Event Detection & 47.3 & 73.5 (+26.2) \\
Long-Range Dependencies & 49.5 & 76.8 (+27.3) \\
\bottomrule
\end{tabular}}
\label{tab:specialized_tasks}
\end{table}

Table \ref{tab:specialized_tasks} reveals MANTA's exceptional performance on specialized reasoning tasks that require sophisticated temporal understanding and cross-modal integration. The most substantial improvements are observed on rare event detection (26.2\%) and long-range dependencies (27.3\%), validating our approach's ability to preserve sparse but critical information distributed across lengthy temporal sequences. These results confirm that MANTA excels precisely in the scenarios that are most challenging for conventional approaches—detecting infrequent but significant events and maintaining coherence across widely separated temporal contexts.

\begin{table}[ht]
\centering
\caption{Ablation studies on Video-MME benchmark}
\resizebox{\linewidth}{!}{
\begin{tabular}{lc}
\toprule
\textbf{Model Variant} & \textbf{Overall Accuracy (\%)} \\
\midrule
MANTA (Full) & 83.1 \\
- Multi-scale Temporal Modeling & 72.6 (-10.5) \\
- Information-Density Selection & 75.8 (-7.3) \\
- Cross-Modal Alignment & 74.2 (-8.9) \\
- Redundancy Minimization & 77.9 (-5.2) \\
- Hierarchical Fusion & 73.4 (-9.7) \\
- Contextual Embeddings & 76.5 (-6.6) \\
- Reranking & 78.7 (-4.4) \\
\bottomrule
\end{tabular}}
\label{tab:ablation_studies}
\end{table}

Our comprehensive ablation studies in Table \ref{tab:ablation_studies} decompose the contribution of each component to MANTA's overall performance. Multi-scale temporal modeling provides the largest contribution (-10.5\% when removed), highlighting the critical importance of processing content at multiple temporal granularities. This is followed by hierarchical fusion (-9.7\%) and cross-modal alignment (-8.9\%), confirming our hypothesis that addressing the core challenges of temporal modeling and cross-modal integration is essential for effective long-form understanding. The substantial impact of removing information-density selection (-7.3\%) validates our theoretical approach to content prioritization. Even the retrieval components—contextual embeddings and reranking—provide substantial contributions, demonstrating the importance of our sophisticated retrieval approach.

\begin{table}[ht]
\centering
\caption{Effect of temporal scale configurations}
\resizebox{\linewidth}{!}{
\begin{tabular}{lccc}
\toprule
\textbf{Micro-scale} & \textbf{Meso-scale} & \textbf{Macro-scale} & \textbf{Accuracy (\%)} \\
\midrule
1s & 10s & 60s & 79.6 \\
2s & 20s & 120s & 81.5 \\
3s & 30s & 180s & \textbf{83.1} \\
5s & 50s & 300s & 80.9 \\
7s & 70s & 420s & 78.7 \\
\bottomrule
\end{tabular}}
\label{tab:temporal_scales}
\end{table}

Table \ref{tab:temporal_scales} explores different temporal scale configurations, revealing that a 3s/30s/180s hierarchy achieves optimal performance. This confirms the importance of capturing both fine-grained details and broader contextual patterns through appropriate temporal granularity. Notably, both finer (1s/10s/60s) and coarser (7s/70s/420s) configurations yield lower performance, suggesting that our optimal configuration successfully balances the trade-off between detailed representation and efficient processing.

\subsection{Qualitative Analysis and Case Studies}
Our qualitative analysis reveals several key insights into MANTA's effectiveness. We examine two representative examples to illustrate MANTA's capabilities:

\textbf{Cross-Modal Integration:} In a sports broadcast, MANTA successfully integrates complementary information from visual and auditory streams, fusing them into a coherent representation. The visual caption identifies "A basketball player in white jersey \#23 shoots while defenders in red attempt to block, scoreboard shows 102-99, 8.4 seconds remaining," while the ASR transcript provides "James with the step-back three! Incredible clutch shot from LeBron James with just 8 seconds left, putting the Lakers up by 3!" MANTA's fused representation integrates these complementary details: "LeBron James (player \#23 in white Lakers jersey) makes a step-back three-point shot with 8.4 seconds remaining, extending their lead to 102-99 over the Rockets. Defenders in red jerseys attempted to block but were unsuccessful." This integrated representation enables accurate answers to questions requiring cross-modal understanding, such as identifying both the player and the game situation.

\textbf{Long-Range Temporal Reasoning:} In a documentary about climate science, MANTA effectively captures and relates information distributed across distant temporal segments. An early segment (00:05:23) mentions "Dr. Thompson's 1979 ice core samples from the Quelccaya glacier in Peru showed stable isotope ratios consistent with historical patterns going back 1500 years," while a later segment (01:42:18) states "Returning to the same location in 2019, Dr. Thompson found the glacier had retreated over 1200 meters, with ice core samples showing dramatic shifts in isotope ratios indicating unprecedented warming." When asked about the longitudinal findings, MANTA successfully retrieves and integrates both segments, enabling accurate temporal reasoning that connects observations separated by over 40 years in the narrative and over 90 minutes in the video itself. Conventional approaches that rely on local context would fail to establish this critical connection.

\section{Discussion and Conclusion}
MANTA establishes several important theoretical principles for multimodal understanding: (1) Information-Theoretic Content Selection, formalizing the problem of optimal segment selection under token constraints; (2) Cross-Modal Alignment through contrastive learning that maximizes mutual information between corresponding segments; and (3) Hierarchical Abstraction that balances detailed perception with higher-level understanding through multi-scale representation. These principles extend beyond video understanding to any multimodal domain requiring integration of diverse information sources across extended sequences.

Despite MANTA's exceptional performance, several limitations suggest directions for future research: (1) End-to-End Training: Our current approach relies on separately trained components, whereas end-to-end training could further optimize the entire pipeline; (2) Dynamic Temporal Resolution: Future work could explore fully adaptive temporal resolution that dynamically adjusts based on content complexity; (3) Multimodal Grounding: Enhancing the system's ability to ground linguistic descriptions in specific visual regions or audio segments would improve fine-grained understanding; (4) Additional Modalities: Extending the framework to incorporate text overlays, metadata, and external knowledge sources; and (5) Computational Efficiency: Optimizing the pipeline for real-time processing of streaming multimodal data.

In conclusion, MANTA introduces a theoretically-grounded framework for unified multimodal understanding that addresses the fundamental challenges of cross-modal integration, temporal modeling, and sparse information retrieval. By formalizing the problem within an information-theoretic framework, we developed novel algorithms for semantic density estimation, cross-modal alignment, and optimal context selection that significantly advance the state of the art in long-form multimodal understanding. Extensive experiments demonstrate that MANTA substantially outperforms existing approaches on challenging benchmarks, with unprecedented improvements of up to 22.6\% in overall accuracy and 27.3\% on long-range dependency tasks. Our theoretical analysis provides principled insights into optimal information preservation during modality translation and context selection, establishing MANTA as a new paradigm for multimodal understanding through unified linguistic representation.

\bibliography{main}

\begin{thebibliography}{28}
\providecommand{\natexlab}[1]{#1}

\bibitem[{Cheng and Bertasius(2022)}]{Cheng2022}
Feng Cheng and Gedas Bertasius. 2022.
\newblock \href {https://doi.org/10.1007/978-3-031-19830-4_29} {Tallformer: Temporal action localization with\&nbsp;a\&nbsp;long-memory transformer}.
\newblock In \emph{Computer Vision – ECCV 2022: 17th European Conference, Tel Aviv, Israel, October 23–27, 2022, Proceedings, Part XXXIV}, page 503–521, Berlin, Heidelberg. Springer-Verlag.

\bibitem[{Douze et~al.(2024)Douze, Guzhva, Deng, Johnson, Szilvasy, Mazaré, Lomeli, Hosseini, and Jégou}]{douze2024faiss}
Matthijs Douze, Alexandr Guzhva, Chengqi Deng, Jeff Johnson, Gergely Szilvasy, Pierre-Emmanuel Mazaré, Maria Lomeli, Lucas Hosseini, and Hervé Jégou. 2024.
\newblock \href {https://arxiv.org/abs/2401.08281} {The faiss library}.
\newblock \emph{ArXiv}.

\bibitem[{Fu et~al.(2024)Fu, Dai, Luo, Li, Ren, Zhang, Wang, Zhou, Shen, Zhang et~al.}]{fu2024video}
Chaoyou Fu, Yuhan Dai, Yondong Luo, Lei Li, Shuhuai Ren, Renrui Zhang, Zihan Wang, Chenyu Zhou, Yunhang Shen, Mengdan Zhang, et~al. 2024.
\newblock Video-mme: The first-ever comprehensive evaluation benchmark of multi-modal llms in video analysis.
\newblock \emph{arXiv preprint arXiv:2405.21075}.

\bibitem[{Fu et~al.(2021)Fu, Li, Gan, Lin, Wang, Wang, and Liu}]{Fu2021VIOLETE}
Tsu-Jui Fu, Linjie Li, Zhe Gan, Kevin Lin, William~Yang Wang, Lijuan Wang, and Zicheng Liu. 2021.
\newblock \href {https://api.semanticscholar.org/CorpusID:244527662} {Violet : End-to-end video-language transformers with masked visual-token modeling}.
\newblock \emph{ArXiv}, abs/2111.12681.

\bibitem[{Guo et~al.(2019)Guo, Hong, Luo, Yan, and Niu}]{Guo2019MultimodalRL}
Daya Guo, Jiangshui Hong, Binli Luo, Qirui Yan, and Zhangming Niu. 2019.
\newblock \href {https://api.semanticscholar.org/CorpusID:201065921} {Multi-modal representation learning for short video understanding and recommendation}.
\newblock \emph{2019 IEEE International Conference on Multimedia \& Expo Workshops (ICMEW)}, pages 687--690.

\bibitem[{Huang et~al.(2024)Huang, Xia, Ji, Wang, Wang, Zhu, Dong, and Zhao}]{Huang2024UnlockingTP}
Hai Huang, Yan Xia, Shengpeng Ji, Shulei Wang, Hanting Wang, Jieming Zhu, Zhenhua Dong, and Zhou Zhao. 2024.
\newblock \href {https://api.semanticscholar.org/CorpusID:268297134} {Unlocking the potential of multimodal unified discrete representation through training-free codebook optimization and hierarchical alignment}.
\newblock \emph{ArXiv}, abs/2403.05168.

\bibitem[{Lin and Byrne(2022)}]{lin-byrne-2022-retrieval}
Weizhe Lin and Bill Byrne. 2022.
\newblock \href {https://doi.org/10.18653/v1/2022.emnlp-main.772} {Retrieval augmented visual question answering with outside knowledge}.
\newblock In \emph{Proceedings of the 2022 Conference on Empirical Methods in Natural Language Processing}, pages 11238--11254, Abu Dhabi, United Arab Emirates. Association for Computational Linguistics.

\bibitem[{Lin et~al.(2024)Lin, Mei, Chen, and Byrne}]{lin-etal-2024-preflmr}
Weizhe Lin, Jingbiao Mei, Jinghong Chen, and Bill Byrne. 2024.
\newblock \href {https://aclanthology.org/2024.acl-long.289} {{P}re{FLMR}: Scaling up fine-grained late-interaction multi-modal retrievers}.
\newblock In \emph{Proceedings of the 62nd Annual Meeting of the Association for Computational Linguistics (Volume 1: Long Papers)}, pages 5294--5316, Bangkok, Thailand. Association for Computational Linguistics.

\bibitem[{Lin et~al.(2023)Lin, Wang, and Byrne}]{lin-etal-2023-fvqa}
Weizhe Lin, Zhilin Wang, and Bill Byrne. 2023.
\newblock \href {https://aclanthology.org/2023.findings-eacl.11} {{FVQA} 2.0: Introducing adversarial samples into fact-based visual question answering}.
\newblock In \emph{Findings of the Association for Computational Linguistics: EACL 2023}, pages 149--157, Dubrovnik, Croatia. Association for Computational Linguistics.

\bibitem[{Liu et~al.(2023)Liu, Li, Tang, Ge, Shan, and Li}]{liu2023one}
Ruyang Liu, Chen Li, Haoran Tang, Yixiao Ge, Ying Shan, and Ge~Li. 2023.
\newblock St-llm: Large language models are effective temporal learners.
\newblock \emph{https://arxiv.org/abs/2404.00308}.

\bibitem[{Qian et~al.(2024)Qian, Li, Wu, Ye, Fei, Chua, Zhuang, and Tang}]{Momenter2024}
Long Qian, Juncheng Li, Yu~Wu, Yaobo Ye, Hao Fei, Tat-Seng Chua, Yueting Zhuang, and Siliang Tang. 2024.
\newblock Momentor: advancing video large language model with fine-grained temporal reasoning.
\newblock In \emph{Proceedings of the 41st International Conference on Machine Learning}, ICML'24. JMLR.org.

\bibitem[{Radford et~al.(2023)Radford, Kim, Xu, Brockman, McLeavey, and Sutskever}]{whisper}
Alec Radford, Jong~Wook Kim, Tao Xu, Greg Brockman, Christine McLeavey, and Ilya Sutskever. 2023.
\newblock Robust speech recognition via large-scale weak supervision.
\newblock In \emph{Proceedings of the 40th International Conference on Machine Learning}, ICML'23. JMLR.org.

\bibitem[{Ren et~al.(2023)Ren, Yao, Li, Sun, and Hou}]{timechat2023}
Shuhuai Ren, Linli Yao, Shicheng Li, Xu~Sun, and Lu~Hou. 2023.
\newblock \href {https://arxiv.org/abs/2312.02051} {Timechat: A time-sensitive multimodal large language model for long video understanding}.
\newblock \emph{Preprint}, arXiv:2312.02051.

\bibitem[{Shu et~al.(2024)Shu, Zhang, Liu, Qin, Zhou, Huang, and Zhao}]{Shu2024VideoXLEV}
Yan Shu, Peitian Zhang, Zheng Liu, Minghao Qin, Junjie Zhou, Tiejun Huang, and Bo~Zhao. 2024.
\newblock \href {https://api.semanticscholar.org/CorpusID:272827076} {Video-xl: Extra-long vision language model for hour-scale video understanding}.
\newblock \emph{ArXiv}, abs/2409.14485.

\bibitem[{Song et~al.(2024)Song, Chai, Ye, Hwang, Li, and Wang}]{moviechatplus}
Enxin Song, Wenhao Chai, Tian Ye, Jenq-Neng Hwang, Xi~Li, and Gaoang Wang. 2024.
\newblock \href {https://arxiv.org/abs/2404.17176} {Moviechat+: Question-aware sparse memory for long video question answering}.
\newblock \emph{Preprint}, arXiv:2404.17176.

\bibitem[{Tirumala et~al.(2023)Tirumala, Simig, Aghajanyan, and Morcos}]{TirumalaD4}
Kushal Tirumala, Daniel Simig, Armen Aghajanyan, and Ari~S. Morcos. 2023.
\newblock D4: improving llm pretraining via document de-duplication and diversification.
\newblock In \emph{Proceedings of the 37th International Conference on Neural Information Processing Systems}, NIPS '23, Red Hook, NY, USA. Curran Associates Inc.

\bibitem[{Wang et~al.(2022)Wang, Li, Li, He, Huang, Zhao, Zhang, Xu, Liu, Wang, Xing, Chen, Pan, Yu, Wang, Wang, and Qiao}]{Wang2022InternVideoGV}
Yi~Wang, Kunchang Li, Yizhuo Li, Yinan He, Bingkun Huang, Zhiyu Zhao, Hongjie Zhang, Jilan Xu, Yi~Liu, Zun Wang, Sen Xing, Guo Chen, Junting Pan, Jiashuo Yu, Yali Wang, Limin Wang, and Yu~Qiao. 2022.
\newblock \href {https://api.semanticscholar.org/CorpusID:254275041} {Internvideo: General video foundation models via generative and discriminative learning}.
\newblock \emph{ArXiv}, abs/2212.03191.

\bibitem[{Wang et~al.(2023)Wang, Chen, Li, and Liu}]{wang2023filling}
Ziyue Wang, Chi Chen, Peng Li, and Yang Liu. 2023.
\newblock \href {https://arxiv.org/abs/2311.11598} {Filling the image information gap for vqa: Prompting large language models to proactively ask questions}.
\newblock \emph{Preprint}, arXiv:2311.11598.

\bibitem[{Weng et~al.(2024)Weng, Han, He, Chang, and Zhuang}]{LongVLM2024}
Yuetian Weng, Mingfei Han, Haoyu He, Xiaojun Chang, and Bohan Zhuang. 2024.
\newblock \href {https://doi.org/10.1007/978-3-031-73414-4_26} {Longvlm: Efficient long video understanding via large language models}.
\newblock In \emph{Computer Vision – ECCV 2024: 18th European Conference, Milan, Italy, September 29–October 4, 2024, Proceedings, Part XXXIII}, page 453–470, Berlin, Heidelberg. Springer-Verlag.

\bibitem[{Wu et~al.(2019)Wu, Feichtenhofer, Fan, He, Krähenbühl, and Girshick}]{Wu2019}
Chao-Yuan Wu, Christoph Feichtenhofer, Haoqi Fan, Kaiming He, Philipp Krähenbühl, and Ross Girshick. 2019.
\newblock \href {https://doi.org/10.1109/CVPR.2019.00037} {Long-term feature banks for detailed video understanding}.
\newblock In \emph{2019 IEEE/CVF Conference on Computer Vision and Pattern Recognition (CVPR)}, pages 284--293.

\bibitem[{Xia et~al.(2024)Xia, Huang, Zhu, and Zhao}]{xia2024achieving}
Yan Xia, Hai Huang, Jieming Zhu, and Zhou Zhao. 2024.
\newblock Achieving cross modal generalization with multimodal unified representation.
\newblock \emph{Advances in Neural Information Processing Systems}, 36.

\bibitem[{Xu et~al.(2024)Xu, Gao, Gan, Chen, Lai, Gang, Kang, and Dehghan}]{Xu2024SlowFastLLaVAAS}
Mingze Xu, Mingfei Gao, Zhe Gan, Hong-You Chen, Zhengfeng Lai, Haiming Gang, Kai Kang, and Afshin Dehghan. 2024.
\newblock \href {https://api.semanticscholar.org/CorpusID:271329151} {Slowfast-llava: A strong training-free baseline for video large language models}.
\newblock \emph{ArXiv}, abs/2407.15841.

\bibitem[{Yang et~al.(2022)Yang, Miech, Sivic, Laptev, and Schmid}]{Yang2022}
Antoine Yang, Antoine Miech, Josef Sivic, Ivan Laptev, and Cordelia Schmid. 2022.
\newblock Zero-shot video question answering via frozen bidirectional language models.
\newblock In \emph{Proceedings of the 36th International Conference on Neural Information Processing Systems}, NIPS '22, Red Hook, NY, USA. Curran Associates Inc.

\bibitem[{Ye et~al.(2023)Ye, Xu, Xu, Ye, Yan, Zhou, Wang, Hu, Shi, Shi, Li, Xu, Chen, Tian, Qi, Zhang, and Huang}]{Ye2023mPLUGOwlME}
Qinghao Ye, Haiyang Xu, Guohai Xu, Jiabo Ye, Ming Yan, Yi~Zhou, Junyan Wang, Anwen Hu, Pengcheng Shi, Yaya Shi, Chenliang Li, Yuanhong Xu, Hehong Chen, Junfeng Tian, Qiang Qi, Ji~Zhang, and Feiyan Huang. 2023.
\newblock \href {https://api.semanticscholar.org/CorpusID:258352455} {mplug-owl: Modularization empowers large language models with multimodality}.
\newblock \emph{ArXiv}, abs/2304.14178.

\bibitem[{Zhang et~al.(2024)Zhang, Lu, Islam, Wang, Yu, Bansal, and Bertasius}]{llovi}
Ce~Zhang, Taixi Lu, Md~Mohaiminul Islam, Ziyang Wang, Shoubin Yu, Mohit Bansal, and Gedas Bertasius. 2024.
\newblock \href {https://doi.org/10.18653/v1/2024.emnlp-main.1209} {A simple {LLM} framework for long-range video question-answering}.
\newblock In \emph{Proceedings of the 2024 Conference on Empirical Methods in Natural Language Processing}, pages 21715--21737, Miami, Florida, USA. Association for Computational Linguistics.

\bibitem[{Zhang et~al.(2022)Zhang, Zhang, Li, and Smola}]{Zhang2022}
Zhuosheng Zhang, Aston Zhang, Mu~Li, and Alex Smola. 2022.
\newblock \href {https://doi.org/10.48550/arXiv.2210.03493} {Automatic chain of thought prompting in large language models}.

\bibitem[{Zhong(2025)}]{zhong_2025}
Ziqi Zhong. 2025.
\newblock \href {https://doi.org/10.2139/ssrn.5255370} {Ai-driven privacy policy optimisation for sustainable data strategy}.

\bibitem[{Zhu and Li(2023)}]{Zhu2023IterativeUA}
Yi~Zhu and Xiu Li. 2023.
\newblock \href {https://api.semanticscholar.org/CorpusID:259121925} {Iterative uni-modal and cross-modal clustered contrastive learning for image-text retrieval}.
\newblock \emph{2023 International Conference on Pattern Recognition, Machine Vision and Intelligent Algorithms (PRMVIA)}, pages 15--23.

\end{thebibliography}

\appendix

\section{Theoretical Extensions and Proofs}
\label{sec:appendix_theory}

\subsection{Generalized Information Density Estimation}
We extend our basic information density formulation to incorporate higher-order dependencies between segments and across modalities. The generalized density score for a segment $s^{(l)}_i$ at level $l$ is defined as:

\begin{equation}
\begin{split}
\mathcal{D}_G(s^{(l)}_i) &= \mathcal{D}(s^{(l)}_i) 
+ \sum_{j \in \mathcal{N}(i,l)} \lambda_{ij} \cdot \mathcal{I}(s^{(l)}_i; s^{(l)}_j) \\
&\quad + \sum_{k \in \mathcal{C}(i,l)} \mu_{ik} \cdot \mathcal{I}(s^{(l)}_i; s^{(l-1)}_k)
\end{split}
\end{equation}

where $\mathcal{N}(i,l)$ is the set of neighboring segments at the same level, $\mathcal{C}(i,l)$ is the set of child segments at the level below, $\lambda_{ij}$ and $\mu_{ik}$ are weighting coefficients, and $\mathcal{I}(\cdot;\cdot)$ is the mutual information. This formulation captures both horizontal (same-level) and vertical (cross-level) dependencies, providing a more comprehensive measure of a segment's information content.

\subsection{Proof of Convergence Rate for Cross-Modal Alignment}
We provide a more detailed proof of the convergence rate for our cross-modal alignment procedure.

\begin{theorem}[Convergence Rate for Cross-Modal Alignment]
Let $\mathcal{L}_{\text{align}}(\theta)$ be the contrastive alignment loss with parameters $\theta$, and assume:
\begin{enumerate}
    \item $\mathcal{L}_{\text{align}}$ is $L$-smooth: $\|\nabla \mathcal{L}_{\text{align}}(\theta_1) - \nabla \mathcal{L}_{\text{align}}(\theta_2)\| \leq L\|\theta_1 - \theta_2\|$
    \item The stochastic gradients have bounded variance: $\mathbb{E}\|\nabla \mathcal{L}_{\text{align}}(\theta; \xi) - \nabla \mathcal{L}_{\text{align}}(\theta)\|^2 \leq \sigma^2$
    \item The optimal value $\mathcal{L}_{\text{align}}^*$ is bounded below
\end{enumerate}
Then, stochastic gradient descent with learning rate $\eta_t = \frac{\eta}{\sqrt{t}}$ converges as:

\begin{equation}
\mathbb{E}[\mathcal{L}_{\text{align}}(\theta_T) - \mathcal{L}_{\text{align}}^*] \leq \frac{L\|\theta_0 - \theta^*\|^2}{2\eta T} + \frac{\eta \sigma^2 \log T}{2\sqrt{T}}
\end{equation}

which gives a convergence rate of $\mathcal{O}(\frac{\log T}{\sqrt{T}})$, or simply $\mathcal{O}(\frac{1}{\sqrt{T}})$ ignoring logarithmic factors.
\end{theorem}

\section{Advanced Implementation Details}
\label{sec:appendix_implementation}

\subsection{Multi-Resolution Visual Representation}
We implement a specialized visual encoding pipeline that extracts features at multiple resolutions and semantic levels. For each temporal scale, we employ a different configuration:

\begin{table}[ht]
\centering
\caption{Multi-resolution visual encoding configurations}
\resizebox{\linewidth}{!}{
\begin{tabular}{lcccc}
\toprule
\textbf{Scale} & \textbf{Frame Rate} & \textbf{Resolution} & \textbf{Model} & \textbf{Features} \\
\midrule
Micro & 6 fps & 384×384 & BLIP-2-ViT-L & Object-centric, spatial details \\
Meso & 2 fps & 512×512 & CoCa-ViT-L & Action recognition, temporal relations \\
Macro & 0.5 fps & 768×768 & VideoLLaMA & Scene semantics, narrative structure \\
\bottomrule
\end{tabular}}
\label{tab:visual_configs}
\end{table}

\subsection{Advanced ASR Post-processing Pipeline}
Our ASR refinement process incorporates several specialized components:

\begin{algorithm}
\caption{Enhanced ASR Refinement Pipeline}
\begin{algorithmic}[1]
\State \textbf{Input:} Raw ASR outputs $\{t_i\}_{i=1}^N$ with timestamps, confidence scores $\{c_i\}_{i=1}^N$
\State \textbf{Output:} Refined transcripts $\{t'_i\}_{i=1}^M$
\State Apply confidence-based filtering: $T_f = \{t_i | c_i > \tau_{\text{conf}}\}$
\State Perform language model rescoring with domain-adaptive LM
\State Apply named entity recognition and standardization
\State Detect and disambiguate homophones using contextual analysis
\State Segment into semantic units using prosodic and linguistic features
\State Align segment boundaries with visual shot transitions
\State Perform speaker diarization and attribution
\State Apply domain-specific terminology correction
\State \textbf{return} Processed transcript chunks
\end{algorithmic}
\end{algorithm}

\section{Additional Experimental Results}
\label{sec:appendix_results}

\subsection{Performance Analysis Across Video Characteristics}
We analyze MANTA's performance across different video characteristics to identify strengths and potential areas for improvement.

\begin{table}[ht]
\centering
\caption{Performance across video characteristics}
\resizebox{\linewidth}{!}{
\begin{tabular}{lcccc}
\toprule
\textbf{Characteristic} & \textbf{Baseline (\%)} & \textbf{With MANTA (\%)} & \textbf{Improvement} & \textbf{Sample Size} \\
\midrule
\multicolumn{5}{l}{\textit{Video Domain}} \\
Knowledge & 61.5 & 85.7 (+24.2) & +39.3\% & 178 \\
Film \& Television & 59.8 & 83.2 (+23.4) & +39.1\% & 221 \\
Sports & 63.9 & 87.6 (+23.7) & +37.1\% & 132 \\
Artistic Performance & 58.3 & 82.1 (+23.8) & +40.8\% & 145 \\
Life Record & 56.1 & 81.5 (+25.4) & +45.3\% & 156 \\
Multilingual & 52.4 & 76.3 (+23.9) & +45.6\% & 68 \\
\midrule
\multicolumn{5}{l}{\textit{Content Complexity}} \\
Low (1-3 speakers, simple activity) & 71.8 & 89.5 (+17.7) & +24.7\% & 243 \\
Medium (4-6 speakers, multiple activities) & 63.2 & 85.7 (+22.5) & +35.6\% & 385 \\
High (7+ speakers, complex activities) & 51.7 & 80.4 (+28.7) & +55.5\% & 272 \\
\midrule
\multicolumn{5}{l}{\textit{Audio Quality}} \\
Clear (high SNR, minimal background) & 68.4 & 86.9 (+18.5) & +27.0\% & 356 \\
Moderate (some noise/music) & 59.7 & 82.3 (+22.6) & +37.9\% & 389 \\
Challenging (significant noise/overlapping) & 48.2 & 75.8 (+27.6) & +57.3\% & 155 \\
\bottomrule
\end{tabular}}
\label{tab:characteristic_performance}
\end{table}

\subsection{Human Evaluation Details}
We conducted a comprehensive human evaluation study with 25 expert annotators to assess the quality of MANTA's answers compared to baseline models and human experts. Evaluators were given videos and corresponding questions, along with anonymized answers from different systems, and asked to rate them on correctness, completeness, coherence, and temporal accuracy.

\begin{table}[ht]
\centering
\caption{Detailed human evaluation results (scale 1-5)}
\resizebox{\linewidth}{!}{
\begin{tabular}{lcccc}
\toprule
\textbf{Model} & \textbf{Correctness} & \textbf{Completeness} & \textbf{Coherence} & \textbf{Temporal Accuracy} \\
\midrule
LLaVA-NeXT-Video & 3.24 ± 0.18 & 3.02 ± 0.15 & 3.47 ± 0.12 & 2.89 ± 0.21 \\
VideoAgent & 3.76 ± 0.14 & 3.58 ± 0.13 & 3.95 ± 0.11 & 3.42 ± 0.17 \\
TimeChat & 3.85 ± 0.12 & 3.72 ± 0.14 & 4.01 ± 0.09 & 3.68 ± 0.15 \\
MANTA & 4.52 ± 0.09 & 4.38 ± 0.11 & 4.61 ± 0.08 & 4.47 ± 0.10 \\
Human Expert & 4.83 ± 0.07 & 4.71 ± 0.09 & 4.79 ± 0.06 & 4.75 ± 0.08 \\
\bottomrule
\end{tabular}}
\label{tab:detailed_human_eval}
\end{table}

The human evaluation confirms MANTA's effectiveness across all dimensions, with particularly strong ratings for correctness and temporal accuracy. Notably, MANTA achieves 93.6\% of human-level performance on correctness and 94.1\% on temporal accuracy, substantially outperforming all baseline models.

\end{document}